\documentclass[journal]{IEEEtran}
\usepackage{amsmath,amsfonts}
\usepackage{algorithmic}
\usepackage{algorithm}
\usepackage{array}
\usepackage{textcomp}
\usepackage{stfloats}
\usepackage{url}
\usepackage{verbatim}
\usepackage{graphicx}
\usepackage[caption=false,font=scriptsize,labelfont=sf,textfont=sf]{subfig} 
\usepackage{cite}

\usepackage[colorlinks=true, allcolors=blue]{hyperref}

\usepackage{amsthm}

\newtheorem{theorem}{Theorem}
\newtheorem{lemma}{Lemma}
\newtheorem{proposition}{Proposition}

\usepackage{enumerate}
\usepackage{multirow}

\begin{document}

\title{High-order Regularization for Machine Learning and Learning-based Control}

\author{Xinghua Liu, Ming Cao

\thanks{The authors are with Institute of Engineering and Technology (ENTEG), the University of Groningen, 9747 AG Groningen, the Netherlands (e-mail: \textit{\{xinghua.liu, m.cao\}}@rug.nl).}
}

\markboth{Journal of \LaTeX\ Class Files,~Vol.~xx, No.~x, xx~xxxx}
{Shell \MakeLowercase{\textit{et al.}}: A Sample Article Using IEEEtran.cls for IEEE Journals}
\maketitle

\begin{abstract}
The paper proposes a novel regularization procedure for machine learning. The proposed high-order regularization (HR) provides new insight into regularization, which is widely used to train a neural network that can be utilized to approximate the action-value function in general reinforcement learning problems. The proposed HR method ensures the provable convergence of the approximation algorithm, which makes the much-needed connection between regularization and explainable learning using neural networks. The proposed HR method theoretically demonstrates that regularization can be regarded as an approximation in terms of inverse mapping with explicitly calculable approximation error, and the $L_2$ regularization is a lower-order case of the proposed method. We provide lower and upper bounds for the error of the proposed HR solution, which helps build a reliable model. We also find that regularization with the proposed HR can be regarded as a contraction. We prove that the generalizability of neural networks can be maximized with a proper regularization matrix, and the proposed HR is applicable for neural networks with any mapping matrix. With the theoretical explanation of the extreme learning machine for neural network training and the proposed high-order regularization, one can better interpret the output of the neural network, thus leading to explainable learning. We present a case study based on regularized extreme learning neural networks to demonstrate the application of the proposed HR and give the corresponding incremental HR solution. We verify the performance of the proposed HR method by solving a classic control problem in reinforcement learning. The result demonstrates the superior performance of the method with significant enhancement in the generalizability of the neural network.
\end{abstract}


\section{Introduction}
Regularization in machine learning is often used to improve the generalizability of a neural network model; a regularization method typically imposes penalties on some properties of the model to avoid overfitting the training data and allow for better generalization to the unseen test data \cite{Goodfellow2016Deep, TIAN2022comprehensive, wu2024deep}. The penalty terms can be designed to reduce the complexity of a given model, and the accordingly obtained regularized model can have similar or even better performance\cite{feng2023activity}. To illustrate this point, we present two examples of overfitting models (blue line) and regularized models (black line) in regression and classification in Figure \ref{model_overfifting}.
\begin{figure}[!ht]
\centering
\includegraphics[width=3.55in]{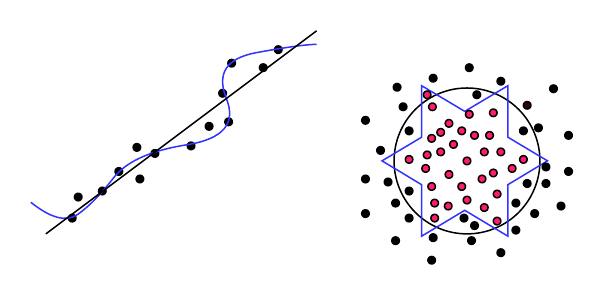}
\caption{Regression problems (left) and classification problems (right) in machine learning. The dots represent the training data samples. The black lines denote regularized models with less complexity while the blue lines are some over-fitting models.}
\label{model_overfifting}
\vspace{0cm} % 降低与下文垂直间距-0.7
\end{figure}

Dropout \cite{wei2020implicit,zhang2024implicit}, Maxup\cite{carratino2022mixup}, early stopping \cite{Goodfellow2016Deep}, $L_2$ norm \cite{lewkowycz2020training} and $L_1$ norm \cite{Goodfellow2016Deep}, batch normalization \cite{wu2018l1}, and manifold regularization \cite{belkin2006manifold} are some of the most popular regularization methods in machine learning. The norm regularizations and their combinations in machine learning
\cite{zou2005regularization,  Goodfellow2016Deep, orvieto2023explicit, lewkowycz2020training, wu2018l1} are the most widely used regularization methods for research and applications \cite{ji2025network, Sun2022Image, Ji2022Adaptive, mignacco2020role, wang2020convolutional}. A better understanding of the norm regularization process is critical to advance on several fronts in machine learning, e.g. in the efforts to improve the reliability and generalizability of neural networks, to ensure convergence of training processes, and to explain learning outcomes.

The training of a neural network tends to approximate the training dataset by adapting the parameters of the model while sometimes overfitting. If the training dataset cannot represent the real environment, the training can be considered an ill-conditioned problem, indicating significant bias in the training dataset. Some regularized methods \cite{deng2009regularized, xu2016incremental} are proposed to improve the generalizability of the neural network \cite{HUANG2006Extreme}. These regularization methods often directly introduce penalty terms to balance the so-called empirical risk and structural risk without giving a theoretical explanation or considering the impact of regularization on the original problem. However, from the perspective of the inverse problem, the introduction of regularization can lead to over-smoothing of the solution \cite{fuhry2012new, yang2015modified,Ji2022Adaptive}. To improve the performance of the neural networks, some incremental methods and back-propagation methods \cite{xu2016incremental, guo2014incremental, zou2017back, zou2022back} for training neural networks are proposed together with these regularization methods, which also impose the demand for the explanation of the regularization in neural networks.

In this paper, we propose a high-order regularization (HR) to improve the performance of the neural network while also providing a theoretical explanation of the regularization. The HR admits the convergence of the trainable weights of a neural network to their optimal solution. We also provide a lower error bound and an upper error bound for the proposed method. With the proposed HR, the error of the output of the neural network trained with acceptable-quality datasets is always bounded. HR is a general regularization that includes a group of $L_2$ regularization methods \cite{zou2005regularization, fuhry2012new, yang2015modified, Sun2022Image}. We show from the experiment results that the proposed HR method enhances the performance of the neural network in a learning-based control problem by improving the generalizability of the model.

\section{Related work}
\subsection{Regularized extreme learning machine}
The extreme learning machine (ELM) is a machine learning algorithm, whose architecture can, in general, be illustrated in Figure \ref{ELM_architeture}. 
\begin{figure}[!ht]
\setlength{\abovecaptionskip}{0cm} 
\setlength{\belowcaptionskip}{0cm}
\centering
\includegraphics[width=3.5in]{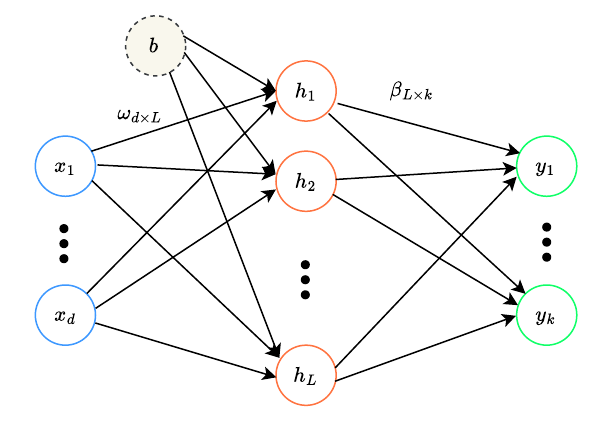}
\caption{The ELM architecture.}
\label{ELM_architeture}
\vspace{0cm}
\end{figure}
The original ELM is a special single-hidden layer feedforward neural network (SLFN) whose hidden nodes are randomly chosen and the output weights are analytically determined 
\cite{HUANG2006Extreme, xu2016incremental, tang2016extreme}. The weights $\omega$ and the (possible) bias $b$ from the input layer to the hidden layer are randomly initiated, and the estimation of the weight $\beta$ from the hidden layer to the output layer of the ELM is determined by the least squares method. The output of the neural network is

\begin{equation}
\sum_{i=1}^L \beta_i g_i\left(\mathrm{x}_j\right)=\sum_{i=1}^L \beta_i g_i\left(w_i \cdot \mathrm{x}_j+b_i\right)=o_j, j=1, \ldots, N,
\end{equation}
where $w_i=\left[\omega_{i, 1}, \ldots, \omega_{i, d}\right]^T, i=1, \ldots L $, is the random weight vector, $\beta_i=\left[\beta_{i, 1}, \ldots, \beta_{i, k}\right]^T, i=1, \ldots L $, is the trainable weight vector, $b=\left[b_i, \ldots, b_L\right] $, is the (possible) bias, $( \cdot )$ represents the inner product, $g=\left[g_i, \ldots, g_L\right] $ is the set of activation functions $g_i$, which, in general, are the same for all hidden nodes, and we omit the subscript for simplicity. 
$\mathrm{o}_j \in \mathrm{R^k}$ is the output of the neural network for a given input, $\left(\mathrm{x}_j, \mathrm{y}_j\right) \in(X, Y), j=1, \ldots N$, is the data sample in the training data set with $\mathrm{x}_j=\left[x_1, \ldots, x_d\right]^T$ and $ \mathrm{y}_j=\left[y_1, \ldots, y_k\right]^T$. 

The extreme learning machine aims to estimate the output weight $\beta$, such that the output error of the neural network is minimal and the neural network can be used for inference. The objective function for the neural network is then
\begin{equation}
\min _{\beta \in \mathrm{R}^{L\times k}}\{\|H \beta-Y\|^2\},
\end{equation}
where $\|\cdot\|$ is the 2-norm if not otherwise specified, $H=\left(h_{j, i}\right)_{N \times L}$ is the hidden layer output matrix with the entry $h_{j, i}=g_i\left(w_i \cdot \mathrm{x}_j+b_i\right)$, the trainable weight $\beta=\left(\beta_{i, m}\right)_{L \times k}$, and $Y=\left(\beta_{j, m}\right)_{N \times k}$ is the label in the training dataset. Note that the hidden layer output matrix represents a linear mapping of the trainable weights, transforming them from the weight parameter space to the data space. This problem can be solved analytically. The solution of the trainable weight $\beta$ can be calculated with the Moore–Penrose (MP) pseudo-inverse as follows:
\begin{equation}\label{problem_1}
\hat \beta =H^{\dagger }Y,
\end{equation}
where $ H^{\dagger }$ is the MP pseudo-inverse of the hidden layer output matrix $H$, and $Y$ represents the output data in the training dataset. $\hat \beta$ is proven to be the smallest norm least
squares solution for $\beta$ \cite{HUANG2006Extreme, wilson2020novel}. Note that the mapping matrix $H$ maps the input in the training data set to a new space with random weights and bias. The total number of the mapping matrix entries indicates the complexity of the neural network. The entries of the matrix $H\beta$ contain the input data information from different paths of the neural network. The layers and nodes of the neural network determine the paths of the network. A study of relationships between the capabilities of neural networks and their layers and nodes indicates that neural networks with the same number of paths admit very similar capabilities \cite{Tamura1997Capabilities}. For a given training data set determining the dimension of the output layer, a given neural network with $L$ hidden nodes provides $L$ paths between each input node and the output layer (excluding the output layer), which is the same as the maximum dimension of the matrix $H^TH$. Hence, the ELM architecture provides a relatively clear explanation for the output of a neural network.

The original ELM tends to overfit and is unstable under an ill-conditioned mapping matrix $H$, especially when the neural network is trained by low-quality training data sets. One method to improve the generalizability of the ELM is regularized ELM, which introduces regularization parameters to regulate the importance of the empirical information measuring the size of the residual and the structural information concerning the size of $\beta$. The regularization problem considering the complexity and structural sparsity is given by \cite{deng2009regularized, xu2016incremental}:
\begin{equation}
\min _{\beta \in \mathrm{R}^{L\times k}} \{\|D_r(H \beta-Y)\|^2_2  + \mu_1 \sum_{i=1}^L |\beta_i| + \bar \mu \|\beta\|_2^2\},
\end{equation}
where $D_r$ is the weight matrix of the residual, $\mu_1$ and $\bar \mu$ are the regularization parameters. As the trainable weights of the neural network connect the mapping matrix $H$ and the output layer, we do not require $\beta$ to be sparse; otherwise, we lose some connections between the mapping matrix and output nodes and lose the usefulness of the corresponding weight parameters in the neural network. One can choose $ \mu_1 = 0$ and $D_r=I$, and due to the fact that there is no prior information in general cases, the $L_2$ regularization solution of trainable weight becomes 
\begin{normalsize}
\begin{equation}\label{RELM_solution}
\hat \beta =\left(H^T H+ \frac{I}{\mu} \right)^{\dagger }H^{T }Y,  \\
\end{equation}
\end{normalsize}
where $ \mu = \frac{1}{\bar \mu } $ is the regularization parameter.

\subsection{Incremental extreme learning machine}
Some incremental learning algorithms for ELM are proposed to improve the learning accuracy \cite{xu2016incremental} by adding new hidden nodes to the neural network in the training process and to update the trainable weights with a minibatch of data \cite{guo2014incremental}. The online update algorithm for the incremental extreme learning machine \cite{guo2014incremental, wilson2020novel} at the time instant $t+1$ after initialization using $N$ samples can be described by 
\begin{align}
\hat \beta_{t+1} &= K_t \hat \beta_{t} + K_t A_{t}^{\dagger }H_{ic}^{T }Y_{ic},\\
A_{t+1}^{\dagger } &= K_t A_{t}^{\dagger },
\end{align}
where
\begin{equation}
K_{t} = I- A_t^{\dagger } H_{ic}^T \left(H_{ic} A_t^{\dagger } H_{ic}^T + I_{n\times n} \right)^{\dagger }H_{ic}, 
\end{equation}
$H_{ic}=\left(h_{j, i}\right)_{n \times L}$ with $h_{j, i}=g\left(w_i \cdot \mathrm{x}_{N+j}+b_i\right) $ are the mapping matrix of $n$ incremental input samples in the new minibatch of data,  $Y_{ic}=\left[\mathrm{y}_{N+1}^T, \cdots, \mathrm{y}_{N+n}^T\right]^T \in R ^{n \times k}$ are the incremental label in the new minibatch of data, and the initial matrix with $N$ samples is denoted by
\begin{normalsize}
\begin{equation}
A_{1} = H^T H+ \frac{I}{\mu}.  \\
\end{equation}
\end{normalsize}

This incremental algorithm updates the trainable weights with a new batch of training data in a smooth manner and avoids dramatic fluctuations in the output of the neural network, compared with the gradient descent-based training process.

\subsection{Learning-based control}
The basic agent-environment interface in reinforcement learning (RL) is shown in Fig \ref{RL_architeture}. As a learning-based control approach, reinforcement learning can obtain its control policy through simulation or real environments with limited online interactions \cite{sutton2018reinforcement}. 
\begin{figure}[!ht]
\centering
\includegraphics[width=3.55in]{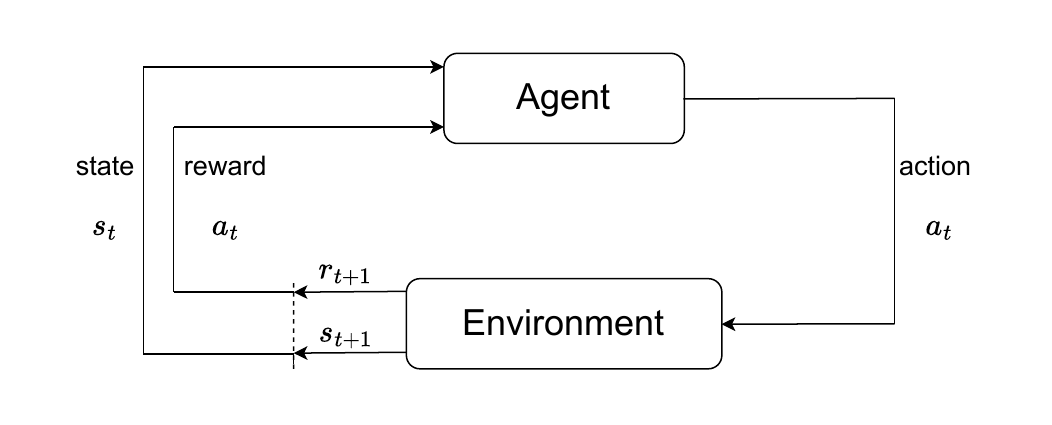}
\caption{The agent-environment interface in RL.}
\label{RL_architeture}
\end{figure}

The objective for the agent in reinforcement learning is to maximize its total discounted reward, which can be solved using the Q-learning method, a model-free RL method. Q-learning maintains an action-value function to find an optimal control policy. The action-value function $Q(S, A)$ of the state $S$ and the action $A$ in Q-learning can be updated using an off-policy temporal-difference (TD) algorithm \cite{sutton2018reinforcement}, which is given by
\begin{equation}
Q(s_{t},a_{t}) \leftarrow Q(s_t,a_t) + \alpha [r_{t+1}+\gamma \max _a Q(s_{t+1},a)-Q(s_t,a_t)], 
\end{equation}
where $s_{t}, a_{t}, r_{t}$ represent the state, action, and reward at the time instant $t$ respectively, $\gamma$ is the discount factor determining the present value of future rewards when executing action $a$,  and $\alpha$ is the step size. The TD error of $ Q(S, A)$ is given by:
\begin{equation}
e = Q(s_t,a_t) - \left(r_{t+1}+\gamma \max _a Q(s_{t+1},a)\right). 
\end{equation}
For safety-critical systems, however, conventional RL cannot guarantee performance, stability, and constraint satisfaction. To approximate $Q$, one can train a Q-network using the mean-squared TD error with a gradient-based update algorithm \cite{wilson2020novel}.

\subsection{Extreme Q-learning machine}
Extreme Q-learning machine (EQLM) \cite{wilson2020novel} alters the neural network training based on the gradient descent algorithm in conventional RL by using a regularized incremental ELM to optimize the trainable weights. The EQLM algorithm is given in Algorithm \ref{alg:EQLM} at the end of the section. To balance exploration and exploitation issues for training an optimal neural network in RL, EQLM implements an $\epsilon$-greedy policy, which is presented as follows:
\begin{equation}
\epsilon= \begin{cases}\epsilon_i-\frac{eps}{N_\epsilon}\left(\epsilon_i-\epsilon_f\right), & \text { if } eps<N_\epsilon \\ \epsilon_f, & \text { if } eps \geq N_\epsilon\end{cases}
\end{equation}
where $eps$ is the index of the episode, $\epsilon_i$ and $\epsilon_f$ are the initial and final values of the exploration rate  $\epsilon$ respectively, dictating the probability of taking random actions, $N_\epsilon$ is the number of episodes to reduce the exploration rate. To improve the stability of the neural network and avoid overestimation of action-values, EQLM maintains a policy target network \cite{mnih2015human} parametrized by $\theta$, and updates this target network by replacing its previous parameters $\theta^-$ after every $C$ time steps. A heuristic policy is also introduced to avoid a sub-optimal policy at the beginning of the training in EQLM, which is given by
\begin{equation}
h_0(t) =mod\left(t,2\right), 
\end{equation}
which requires the agent to execute alternating actions at each time step during the first few episodes. These policies are denoted by $\pi$. 

An experience replay strategy to train the neural network is employed in EQLM to improve the data efficiency, which utilizes the experiences of transitions stored in a memory $D$. The memory $D$ contains action-state-reward tuples $\{s_t,a_t,r_{t},s_{t+1}\}$ for each transition. To balance the usage of memory and data efficiency, a sliding window of the width $N_{mem}$ is used to update the memory $D$. For each update of the neural network, $n$ experiences are selected randomly from the current memory $D$ to construct a minibatch of the training data.

 \begin{algorithm}[!ht]
\caption{Extreme Q-learning machine (EQLM)}\label{alg:EQLM}
\begin{algorithmic}
\STATE 
\STATE 1 Initialize the network with random weights and bias
\STATE 2 Set $step0 \leftarrow$ $True$
\STATE 3 \textbf{For} episode=1 to maximum episode \textbf{do}
\STATE 4 \hspace{0.1cm} Initialize state $s_t \leftarrow s_0$
\STATE 5 \hspace{0.1cm} \textbf{While} state $s_t$ is non-terminal \textbf{do}
\STATE 6 \hspace{0.4cm} \textbf{If} episode $\leq N_h$, \textbf{then}
\STATE 7 \hspace{0.4cm} Select action $a_t$, according to heuristic $h_0(t)$
\STATE 8 \hspace{0.4cm} \textbf{Else}
\STATE 9 \hspace{0.4cm} Select action $a_t$, according to policy $\pi$
\STATE 10 \hspace{0.25cm} \textbf{End if}
\STATE 11 \hspace{0.25cm} Execute action $a_t$, observe $r, s_{t+1}$
\STATE 12 \hspace{0.25cm} Update memory $D$ with $\left(s_t, a_t, r_t, s_{t+1} \right)$
\STATE 13 \hspace{0.25cm} Select random minibatch of $n$ experiences $\left(s_t, a_t, r_t, s_{t+1} \right)$ 
\STATE   \hspace{0.68cm} from $D$
\STATE 14 \hspace{0.3cm} $\mathbf{y}_j= \begin{cases}r_j, & \text { if } s_{j+1} \text { is terminal } \\ r_j+\gamma \max _a Q\left(s_{j+1}, a\right), & \text { otherwise }\end{cases}$
\STATE 15 \hspace{0.25cm} Construct matrix $H$
\STATE 16 \hspace{0.25cm} \textbf{If} $step0==$ {\text{True}} \textbf{then}
\STATE 17 \hspace{0.7cm}$ A_{1} = H^T H+ \frac{I}{\mu}$
\STATE 18 \hspace{0.7cm}$ \hat \beta{_{t}}=A_{1}^{\dagger }H^{T }Y$
\STATE 19 \hspace{0.7cm}$ A_{t} = A_{1}$
\STATE 20 \hspace{0.65cm} Set $step0 \leftarrow$ $False$
\STATE 21 \hspace{0.25cm} \textbf{Else}
\STATE 22 \hspace{0.7cm}$ K_{t} = I- A_t^{\dagger } H^T \left(H A_t^{\dagger } H^T + I_{n\times n} \right)^{\dagger }H$
\STATE 23 \hspace{0.7cm}$ \hat \beta_{t+1} = K_t \hat \beta_{t} + K_t A_{t}^{\dagger }H^{T }Y$
\STATE 24 \hspace{0.7cm}$ A_{t+1}^{\dagger } = K_{t}A_{t}^{\dagger }$
\STATE 25 \hspace{0.25cm} \textbf{End if}
\STATE 26 \hspace{0.25cm} After $C$ time-steps set $\theta^- \leftarrow \theta$
\STATE 27 \hspace{0.1cm} \textbf{End while}
\STATE 28 \hspace{-0.1cm} \textbf{End for}
\end{algorithmic}
\end{algorithm}

The EQLM employs the regularized method directly without considering the ill-condition of the problem caused by the possible low quality of the training dataset at the beginning of the training, which can lead the neural network to overfit a biased dataset, resulting in an inappropriate approximation of the action-value function. Therefore, one can improve the performance of the neural network by mitigating the overfitting problem in ELM.

\section{High-order regularization} 
\subsection{High-order regularization for machine learning}
To mitigate the overfitting of ELM, we propose a high-order regularization (HR) method to determine the trainable weights. Take \begin{equation} \label{regularization_factor}
F(R):=R\left(H^T H+R\right)^{-1},
\end{equation}
then the information matrix $H^T H$ can be rewritten as
\begin{equation}
H^T H=\left(I-F(R)\right)\left(H^T H+R\right), \end{equation}
whose inverse is
\begin{equation}
(H^T H)^{-1}=\left(H^T H+R\right)^{-1}\left(I-F(R)\right)^{-1}, \end{equation}
if it is invertible, and recall the matrix power series \cite{horn2012matrix}
\begin{equation}\label{matrix_power_series}
\begin{aligned}
& \left(I-F(R)\right)^{-1}=\sum_{i=0}^{\infty}\left(F^i(R)\right), &\text { if } \rho\left(F(R)\right)<1, \\
\end{aligned}
\end{equation}
where $\rho\left( \cdot \right)$ is the spectral radius of a matrix. By truncating the terms after $c+1$ items, we obtain the proposed high-order regularization method
\begin{equation} \label{hrm_1}
\hat \beta{_{h r}}=\left(H^T H+R\right)^{-1} \sum_{i=0}^c F^i(R) H^T Y , \rho\left(F(R)\right)<1,
\end{equation}
where $R$ is the regularization matrix, and $c$ is the regularization order, a hyperparameter with $c=1$ in general. It can be checked that the $L_2$ regularization solution (\ref{RELM_solution}) of the trainable weights is a special case of the HR  solution for $c=0$ with $R = \frac{1}{\mu}I$. The HR solution mitigates over-smoothness by approximating the inverse of $H^T H$ with more than one term of the matrix series of $F(R)$, which improves the generalizability and stability of the neural network, allowing performance improvement.

Note that if $R \rightarrow O$, then $F(R)\rightarrow O$, and $\|F(R)\|\rightarrow 0$, where $O$ denotes the zero matrix. One can also check that if $R$ is symmetric positive semidefinite (SPD), denoted by $R \in S_+^n$, and $H^T H$ is positive definite (PD), denoted by $H^T H \in S_{++}^n$, then $\rho\left(F(R)\right)<1$ is satisfied \cite{horn2012matrix}, and the condition holds as long as one of the two SPD matrices $H^T H$ and $R$ is positive definite. In the rest of this section, we assume these conditions always hold.

\subsection{Convergence of the high-order regularization method}
Suppose there is an optimal solution for trainable weights \cite{HUANG2006Extreme,feng2009error}, then the error of the proposed HR method is given by
\begin{equation} \label{regularization_error}
e_\beta = \hat \beta{_{h r}} - \beta{_{opt}},  
\end{equation}
where the optimal solution $\beta{_{opt}}$ is estimated by the MP pseudo-inverse. Denote the truncated terms of the matrix series by
\begin{equation}
F_{tt}(R):=\sum_{i=c+1}^{\infty}F^i(R)
=\left(I-F(R)\right)^{-1}F^{c+1}(R),
\end{equation}
then, the truncated terms converge
\begin{equation}
\begin{aligned}
\lim _{R \rightarrow O}F_{tt}(R)
&= O, \lim _{c \rightarrow +\infty}F_{tt}(R)=O.
\end{aligned}
\end{equation}
We now prove that the error of the proposed HR solution converges to zero.
\begin{theorem} \label{convergence}
When $c\geq 0$, $(H^T H)\in S_{++}^n$, $\rho\left(F(R)\right)<1$, and the error $e_\beta$ defined by equation \eqref{regularization_error}, it holds that
\begin{equation}
\begin{aligned}
\lim _{R \rightarrow O}\left\|e_\beta\right\|
& = \lim _{R \rightarrow O}\left\|\left(H^T H+R\right)^{-1} F(R) H^T Y\right\| =0,
\end{aligned}
\end{equation}
and
\begin{equation}
\begin{aligned}
\lim _{c \rightarrow +\infty}\left\|e_\beta\right\| =0.
\end{aligned}
\end{equation}
\end{theorem}
\begin{proof}
Under the required spectral radius condition, it follows that when $R \rightarrow O$, then $F(R)\rightarrow O$ and $I-F(R)\rightarrow I$, and one can denote the approximation residual of the matrix inverse of $(H^TH)$ by
\begin{equation}
\begin{aligned}
F_{ar}(R)& =
\left(H^T H\right)^{-1}-\left(H^T H+R\right)^{-1} \sum_{i=0}^c F^i(R) \\
& =\left(H^T H+R\right)^{-1}\left(I-F(R)\right)^{-1}F^{c+1}(R)\\
& =\left(H^T H\right)^{-1}F^{c+1}(R).\\
\end{aligned}
\end{equation} 
Now we submit it into the error expression, and take the norm
\begin{equation}
\begin{aligned}
\lim _{R \rightarrow O}\left\|e_\beta\right\|
&= \lim _{R \rightarrow O}\left\|\left(H^T H+R\right)^{-1} \sum_{i=c+1}^{\infty}F^i(R) H^T Y\right\| \\
&= \lim _{R \rightarrow O}\left\|\left(H^T H+R\right)^{-1} F(R) H^T Y\right\| =0, 
\end{aligned}
\end{equation}
\begin{equation}
\lim _{c \rightarrow +\infty}\left\|e_\beta\right\|=0,
\end{equation}
and we arrive at the conclusion.
\end{proof}
These properties of the HR solution ensure that the result of the trainable weights converges to an optimal solution and has a bounded error under a given regularization matrix, which is difficult to guarantee by gradient descent-based algorithms.

\subsection{Error bounds of high-order regularization}
The proposed HR solution also ensures a lower error bound and an upper error bound for a given regularization matrix $R$. For a given $c$, one has
\begin{equation}
\begin{aligned}
\frac{1}{1-\lambda_{min }(F(R))}\left\|\left(H^T H+R\right)^{-1}F^{c+1}(R)H^T Y \right\| \leq\left\|e_\beta\right\|,
\end{aligned}
\end{equation}
\begin{equation}
\begin{aligned}
\left\|e_\beta\right\| \leq \frac{1}{1-\lambda_{max }(F(R))}\left\|\left(H^T H+R\right)^{-1}F^{c+1}(R)H^T Y \right\|,
\end{aligned}
\end{equation}
where $\lambda_{min}(F(R))$ and $ \lambda_{max}(F(R))$ are the minimum and maximum eigenvalues of $F(R)$ respectively. 
\begin{proof}
See Appendix \ref{two_bounds_of_regularization_error}.
\end{proof}

The upper error bound indicates that the HR solution is bounded by the maximum eigenvalue of $F(R)$, while the lower error bound can be zero. The upper error bound also provides insight into how to select the regularization matrix. Note that for $H^T H\in S_{++}^n$, one has
\begin{equation}
\begin{aligned}
F_{ar}(R)& =
\left(H^T H\right)^{-1}F^{c+1}(R), \\
\end{aligned}
\end{equation} 
then, for a given $c$, the bound of the estimation bias is
\begin{equation} \label{estimation_error_bound}
\begin{aligned}
\left\|e_\beta\right\|=\left\|\left(H^T H\right)^{-1}F^{c+1}(R)H^T Y \right\|.
\end{aligned}
\end{equation}
One also has the following conclusion.
\begin{lemma} \label{convex_regularization_factor}
If $R \in S_+^n$ and $H^T H \in S_{++}^n$ commute, then $F(R)$ defined by \eqref{regularization_factor} is a convex function with respect to $R$.
\end{lemma}
\begin{proof}
See Appendix \ref{convexity_of_factor_F}.
\end{proof}

Following Lemma \ref{convex_regularization_factor}, $(I-F(R))^{-1}$ and $\left\|F(R)\right\|$ are convex functions\cite{boyd2004convex,scherer2000linear}.

\begin{theorem} \label{error_bound_contraction}
For $R\in S_+^n, F(R)\in S_+^n$ and $H^T H \in S_{++}^n$, it holds that $R$ admits a contraction of $\left\|F_{ar}(R)\right\|$, and the contraction rate is determined by the regularization order $c$ and $R$.
\end{theorem}
\begin{proof}
Since $F(R)\in S_+^n$, we have that $R\in S_+^n$ and $(H^T H+R)^{-1} \in S_{++}^n$ commute, and can be simultaneously diagonalized by some unitary matrix. As $(H^T H)^{-1}$ and $H^T H$ always commute, take
\begin{equation} \label{commuting_matrix_eigenvlaue}
H^TH=P \operatorname{diag}\left\{\lambda_{ 1}, \ldots, \lambda_{n}\right\} P^{-1}
\end{equation}
and
\begin{equation} \label{selecion_regularization_matrix}
R=P \operatorname{diag}\left\{\lambda_{R, 1}, \ldots, \lambda_{R, n}\right\} P^{-1}
\end{equation}
where $\lambda_{i} > 0, i=1, \ldots, n$ is eigenvalues of $H^TH$, $\lambda_{i} \geq \lambda_{i+1}$, $\lambda_{R, i} \geq 0$ is eigenvalues of $R$, and $P$ is a unitary matrix whose columns are the eigenvectors of matrix $H^T H$. Then one has
\begin{equation}
\begin{aligned}
\left\| F_{ar}(R)\right\|&= \lambda_{max}(F_{ar}(R)) \\
&= \max{ \left\{\frac{\lambda_{R,i}^{c+1}}{\lambda_{i}(\lambda_{i}+\lambda_{R,i})^{c+1}}\right\}}, i =1, \cdots, n, \\
&:=\max{ \left\{h_{ar}(g_{ar}(\lambda_{R,i}))\right\}}, i =1, \cdots, n,
\end{aligned}
\end{equation}
where $h_{ar}(g_{ar})=\frac{1}{\lambda_{i}}g_{ar}^{c+1}, g_{ar}(\lambda_{R,i})=\frac{{\lambda_{R,i}}}{\lambda_{i}+{\lambda_{R,i}}}$, and $c\geq 0$. Note that $\rho\left(F(R)\right)<1$ is always satisfied, then $g_{ar}(\lambda_{R,i})<1$, and
\begin{equation}
\begin{aligned}
\left\| F_{ar}(R)\right\|& \leq \max{ \left\{\frac{1}{\lambda_{i}}\right\}}\max{ \left\{g_{ar}^{c+1}\right\}} < \frac{1}{\lambda_{n}}.
\end{aligned}
\end{equation}
Hence, the regularization order $c$ and $R$ define a contraction of the approximation residual $F_{ar}(R)$ and $\left\| F_{ar}(R)\right\|$.
\end{proof}
From Theorem \ref{error_bound_contraction}, for a given hidden layer output matrix $H$, the regularization order $c$ and the regularization matrix $R$ determine the contraction rate of $\left\| F_{ar}(R)\right\|$, which characterizes the error bound of the HR solution. Note that the truncated terms $F_{tt}(R)$ are the direct reason for the bias of regularization solutions. For a given $c$, the estimation error can be calculated.

\subsection{Maximum generalizability and optimal regularization matrix}
As the condition number of $H^TH$ can be used to describe the ill condition of a problem, the over-fitting or the generalizability of the proposed HR neural network can be characterized by the condition number of $H^TH+R$. Denote the condition number of the new matrix $H^TH+R$ by
\begin{equation} \label{condition_number}
 \begin{aligned}
\text{Cond}(R):=&\left\|(H^TH+R)\right\|\left\|(H^TH+R)^{-1}\right\|,
\end{aligned}
\end{equation}
and to determine the optimal regularization matrix $R$ that balances the bias and the over-fitting problems, a natural approach is to solve the following optimization problem 
\begin{equation} \label{criterion_error_bound_cond}
\min _{R \in S_+^n}\text{Obj}(R):= \min _{R \in S_+^n} \left\{\left\|F_{ar}(R) \right\| \cdot \text{Cond}(R)\right\}.
\end{equation}
This criterion $\text{Obj}(R)$ balances the estimation error and the over-fitting problem while estimating the trainable weights. One has the following conclusion.

\begin{proposition} \label{TR_HR_regularization_parameter}
For a given $H^T H \in S_{++}^n$ and $R= {\bar \mu}I \in S_{++}^n$, it holds that $\text{Cond}(R)$ is a bounded decreasing function; $\left\|F_{ar}(R) \right\|$ is a bounded increasing function; and problem \eqref{criterion_error_bound_cond} 
has a global minimum.
\end{proposition}
\begin{proof}
For $R= {\bar \mu}I$, from the proof of Theorem \ref{error_bound_contraction}, one has $g_{ar}(\lambda_{R,i})=\frac{{\bar \mu}}{\lambda_{i}+{\bar \mu}}$, and $\frac{h_{ar}(g_{ar}(\lambda_{R,i}))} {h_{ar}(g_{ar}(\lambda_{R,i+1}))} \leq1$, and then
\begin{equation}
\begin{aligned}
\left\|F_{ar}(R)\right\|
&=\frac{1}{\lambda_{n}}\left(\frac{{\bar \mu}}{\lambda_{n}+{\bar \mu}}\right)^{c+1}
:=h_{ar}(g_{ar}(\bar \mu)),
\end{aligned}
\end{equation}
where $g_{ar}(\bar \mu)=\frac{{\bar \mu}}{\lambda_{n}+{\bar \mu}}=1-\frac{\lambda_{n}}{\lambda_{n}+{\bar \mu}}$, and $h_{ar}(g_{ar})=\frac{1}{\lambda_{n}}g_{ar}^{c+1}$. By checking its first derivative with respect to $\bar \mu$, it is an increasing function, and we further have
\begin{equation}
\lim _{\bar \mu \rightarrow +\infty}\left\|F_{ar}(R)\right\|=\frac{1}{\lambda_{n}},
\end{equation}
and $\left\|F_{ar}(R)\right\|\in \left[0,\frac{1}{\lambda_{n}}\right]$.
Note that it holds that 
\begin{equation}
\begin{aligned}
\text{Cond}(R) &=\frac{\lambda_{max}(H^TH+R)}{\lambda_{min}(H^TH+R)} \\
&=\frac{\lambda_{1}+{\bar \mu}}{\lambda_{n}+{\bar \mu}} \in \left[1, 1+\frac{\lambda_{1}-\lambda_{n}}{\lambda_{n}}\right],
\end{aligned}
\end{equation}
which is a bounded decreasing function, and the criterion 
\begin{equation}
\begin{aligned}
\text{Obj}(R) &= \left\|F_{ar}(R)\right\| \cdot \text{Cond}(R) \\
&= 
h_{ar}(g_{ar}(\bar \mu)) \left(1+\frac{\lambda_{1}-\lambda_{n}}{\lambda_{n}+{\bar \mu}}\right). 
% \frac{{\bar \mu}^{c+1}}{\lambda_{n}(\lambda_{n}+{\bar \mu})^{c+1}}+\frac{\lambda_{1}+{\bar \mu}}{\lambda_{n}+{\bar \mu}}.
\end{aligned}
\end{equation}
Hence, there must be a global minimum for problem \eqref{criterion_error_bound_cond}.
\end{proof}

Proposition \ref{TR_HR_regularization_parameter} indicates that there is always a minimum for the problem \eqref{criterion_error_bound_cond} with a given $H^TH$. However, for the regularization matrix selection $R= \frac{1}{\mu}I$, the maximal generalizability is realized at $R\rightarrow +\infty (\frac{1}{\mu} \rightarrow +\infty)$, in which case one has $F_{ar}(R)\rightarrow I$ and $e_{\beta} \rightarrow \beta_{opt}$, which means that the maximal generalizability is realized by ignoring the effect of current data. This situation occurs because the selection of the regularization matrix in Proposition \ref{TR_HR_regularization_parameter} changes all of the eigenvalues in the information matrix $H^TH$ using the same parameter, thus negating all information used to estimate the variables of the neural network. This result indicates that the regularization matrix selection is important for some problems, in which the neural networks are trained with reliable data. Additionally, the importance of estimation bias and generalizability are equally emphasized in problem \eqref{criterion_error_bound_cond}; however, the optimal balance should be determined by specific use cases. In these situations, one can find a more proper selection of the regularization matrix or a more effective criterion to determine the regularization matrix.

There is a way to preserve the most significant information corresponding to the maximum eigenvalue of $H^TH$. Denote $\bar \mu= \frac{1}{\mu}$, and one has the following conclusion.

\begin{proposition} \label{criterion_keep_most_important_information}
For $R\in S_+^n, F(R)\in S_+^n$, $H^T H \in S_{++}^n$ and $\lambda_{R, i} = \max{\left\{\bar \mu - \lambda_i, 0\right\}}, \bar \mu  \geq \lambda_2$, it holds that $\text{Obj}(R)$ defined in \eqref{criterion_error_bound_cond} is a bounded function and has a global minimum; and $\text{Cond}(R)$ admits the maximization of generalizability of neural networks at $\bar \mu \geq \lambda_1$; and if $\lambda_1 \geq \bar \mu  \geq\lambda_2$, the new matrix $H^TH+R$ preserves the maximum eigenvalue of $H^T H$.
\end{proposition}
\begin{proof}

For a given $\bar \mu$, since $\lambda_{R, i} = \max{\left\{\bar \mu - \lambda_i, 0\right\}}$ and $\lambda_1 \geq \bar \mu  \geq\lambda_2$, it follows that
$\lambda_1 +\lambda_{R, 1} = \lambda_1$, and $\frac{\lambda_{R, i}}{\lambda_{R, i+1}} \leq 1$. Thus, one has
\begin{equation}
\begin{aligned}
\left\| F_{ar}(R)\right\|= \frac{(\bar \mu - \lambda_n)^{c+1}}{\lambda_{n}\bar \mu^{c+1}},
\end{aligned}
\end{equation}
and 
\begin{equation}
\text{Cond}(R)
 = \frac{\lambda_{1}+\lambda_{R, 1}}{\bar \mu} = \frac{\lambda_{1}}{\bar \mu}.
\end{equation}
These two functions are both bounded. The minimum of $\text{Cond}(R)$ is found at
\begin{equation}
\bar \mu=\lambda_1,
\end{equation}
and the condition number of the new matrix $H^TH+R$ is minimized as
\begin{equation}
\text{Cond}(R)=1,
\end{equation}
which maximizes the generalizability of neural networks.
\end{proof}

Characterizing the generalizability of neural networks by the condition number of the corresponding information matrix, then for $\bar \mu=\lambda_1$, the generalizability of neural networks is maximized as $\text{Cond}(R)$ is minimized, which also indicates the learning ability of the neural network is maximized. For $\bar \mu \geq \lambda_1$, one can search the minimum of $F_{ar}(R)$ by changing the regularization matrix, while preserving the minimal generalizability. The obtained regularization matrix balances the estimation error and generalizability while preserving the most significant information of the training data. As for ill-conditioned inverse problems, $\lambda_1 \gg \lambda_n$, then for simplicity, one can take $\lambda_{R, 1}=0$ and $\lambda_{R, i} = \lambda_1, i\neq1$, or further simplify as $R = \lambda_1 I$, which saves much computing power as only $\lambda_1$ is required. These simplifications can be very useful for situations where $H$ is very large\cite{golub2013matrix}.

\subsection{Neural network scaling}
When scaling up the size of neural networks, $H^TH$ can be a singular SPD matrix but not a PD matrix in some problems. In these cases, considering the nonzero eigenvalues of $H^T H$ or the singular values of $H$ is useful, and one can utilize a PD regularization matrix $R$ and swap $H^T H$ and $R$ to obtain similar conclusions. For example, by replacing the roles of $H^T H$ and $R$ in $F(R)$, so that $F(R) = H^T H(H^T H + R)^{-1}$, the equation \eqref{matrix_power_series} still holds. Consider the following objective function for the neural network
\begin{equation}
\min _{\beta \in \mathrm{R}^{L\times k}}\{\|H^TH \beta-H^TY\|^2\}.
\end{equation}
Then the estimation error of the proposed HR method is given by
\begin{equation}\label{hrm_1_singular_HtH}
e_\beta=\left((H^TH)^{\dagger}-F_{aim}(H,R,c)\right) H^T Y
\end{equation}
where $F_{aim}(H,R,c)$ is the approximation of the (general) inverse of the mapping $H^TH$, namely
\begin{equation}\label{approximation_inverse_mapping}
F_{aim}(H,R,c):=\left(H^T H+R\right)^{-1} \sum_{i=0}^c F^i(R),
\end{equation}
and the approximation residual of the matrix inverse of $(H^TH)$ can be obtained by
\begin{equation}
\begin{aligned}
F_{ar}(R)&=(H^TH)^{\dagger}-F_{aim}(H,R,c).
\end{aligned}
\end{equation}
One has the following conclusion.
\begin{theorem}\label{regularization_singular_maping}
For $H^T H \in S_{+}^n$ and $R\in S_{++}^n$, it holds that the approximation residual $F_{ar}(R)$ has
\begin{equation}
\begin{aligned}
\lim _{c \rightarrow +\infty}F_{ar}(R)&=(H^TH)^{\dagger}-R^{-1}.
\end{aligned}
\end{equation}
\end{theorem} 
\begin{proof}
The conclusion follows directly from the fact that $\sum_{i=0}^{+\infty} F_{aim}(H,R,c)=R^{-1}$. 
\end{proof}
Theorem \ref{regularization_singular_maping} indicates that regularization in machine learning is an approximation in terms of the general inverse mapping. It also provides the insight into how to select a regularization matrix and allows an approximation residual to be small by setting $\|(H^TH)^{\dagger}-R^{-1}\|\rightarrow 0$, which admits a small estimation bias of regularization solutions, while pursuing generalizability enhancement.

\begin{proposition}
\label{regularization_singular_maping_mu}
For $H^T H \in S_{+}^n$, (a) if $R= {\bar \mu}I$, then $\text{Cond}(R)$ is a decreasing function with respect to $\bar \mu$; and (b) if $\lambda_{R,i}= {\bar \mu}-\lambda_{i}, \bar \mu >\lambda_{1}$, then the generalizability of the neural network is maximized.
\end{proposition} 
\begin{proof}
Similar to Proposition \ref{TR_HR_regularization_parameter} and Proposition \ref{criterion_keep_most_important_information}, if $R= {\bar \mu}I$, one has
\begin{equation}
\begin{aligned}
\text{Cond}(R)
&=1+\frac{\lambda_{1}}{{\bar \mu}} \in \left(1,+\infty\right),
\end{aligned}
\end{equation}
which is a decreasing function concerning $\bar \mu$, and
if $\lambda_{R,i}= {\bar \mu}-\lambda_{i}, \bar \mu >\lambda_{1}$, one has
\begin{equation}
\begin{aligned}
H^TH+R= \bar \mu I, \text{Cond}(R)=1.
\end{aligned}
\end{equation}
Then the proof is complete.
\end{proof}

Proposition \ref{regularization_singular_maping_mu} indicates that the generalizability of the neural networks can be maximized by adjusting the regularization matrix $R$, similar to Proposition \ref{TR_HR_regularization_parameter} and Proposition \ref{criterion_keep_most_important_information}. Note that one has
\begin{equation}
\begin{aligned}
F_{ar}(R)
&=P\left(\Sigma^{\dagger}-(\Sigma+\Sigma_R)^{\dagger}\sum_{i=0}^c \left(\Sigma (\Sigma+\Sigma_R)^{\dagger}\right)\right)P^{-1},
\end{aligned}
\end{equation}
where $\Sigma$ and $\Sigma_R$ are diagonal matrices constructed by the singular values of $H^TH$ and $R$, respectively. Hence, for $R= {\bar \mu}I$, ${\bar \mu}$ cannot be too small; otherwise, the estimation bias can be significantly large. Note that for a given neural network, there is always some $R$ such that $H^TH+R$ is sparse and $\rho\left(F(R)\right)<1$, which guarantees that the equation \eqref{matrix_power_series} holds. Therefore, the proposed HR is always applicable for any mapping $H$. Even if some regularization matrix $R$ does not satisfy $\rho\left(F(R)\right)<1$ but only ensures a bounded spectral radius, the estimation bias is also bounded due to a finite given order $c$, making these $F(R)$ still meaningful.

According to Theorem \ref{regularization_singular_maping}, there are many meaningful approximations of the general inverse mapping depending on the selected form of the regularization matrix. One can consider a selection of the regularization matrix as 
$R^{-1}=(H^TH)^{\dagger}+\Sigma_{n-k}+{\mu}I$, where $\Sigma_{n-k}$ is a diagonal matrix with the rank $n-k$ and $k=rank(H^TH)$. Then one has the approximation residual $\lim _{c \rightarrow +\infty}F_{ar}(R)= \Sigma_{n-k}+{\mu}I$, which is determined only by $R$ and can converge to zero. Following this selection, there is another selection for $R$ which is determined by letting the estimation bias $F_{ar}(R)= {\mu}I$. Most of the above selections of $R$ reach a similar performance while $\mu\rightarrow 0$, i.e., $\bar \mu\rightarrow +\infty$, since they approach the maximum generalizability while $\bar \mu\rightarrow +\infty$. The relationships between $\text{Obj}(R)$ and $\bar \mu$ for two selections of $R$ are presented in Fig. \ref{relationship_ObjR_mu}. According to the illustrations of different selections of $R$, $\text{Obj}(R)$ shows similar behavior when $\bar \mu$ varies. The results of the regularization order $c\geq1$ are almost the same, while they are different from $c=0$, especially while $\bar \mu$ is small. This result indicates that the regularization order $c=1$ is sufficient in some cases, while a general way to determine the regularization order $c$ relies on integer optimization or mixed-integer programming.

\begin{figure}[!ht]
\centering
\includegraphics[width=3.6in]{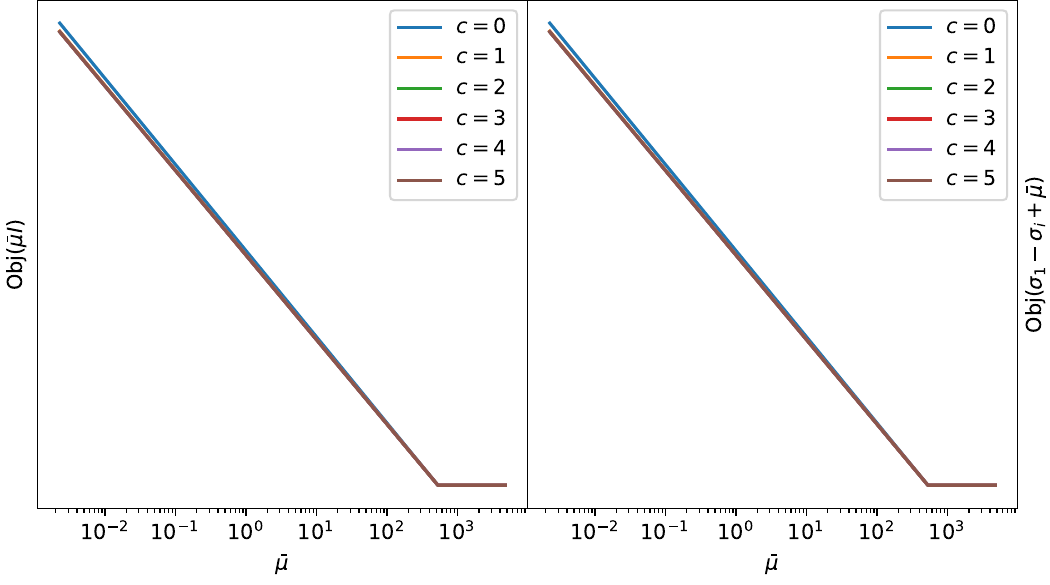}
\caption{The relationship between $\text{Obj}(R)$ and $\bar \mu$ for some selections. $R=\bar\mu$ in the left figure, and $\lambda_{R,i}= \sigma_{1} -\sigma_{i} +\bar\mu$ in the right, where $\sigma_{i}(\sigma_{i}\geq\sigma_{i+1},i=1,\cdots, n)$ is the singular value of $H^TH$. There are different objective functions $\text{Obj}(R)$ with respect to different regularization order $c\in \{0,\cdots, 5\}$ for the two selections of $R$.}
\label{relationship_ObjR_mu}
\end{figure}

\subsection{Explanation of regularization}
% Regularization is generally and empirically
Regularization can be used to restrict the complexity of variables and enhance structural sparsity. We are now ready to provide a novel, full theoretical explanation of regularization. From equation \eqref{matrix_power_series}, a theoretical explanation of regularization is that it is an approximation of the inverse mapping $\left(I-F(R)\right)^{-1}$. With the proposed HR method, the truncated terms determine the calculable bias between the theoretical optimal and regularized solutions with improved generalizability. The approximation of the inverse of $H^TH$ is given by $F_{aim}(H,R,c)$. 
\begin{figure}[!ht]
\centering
\includegraphics[width=2.5in]{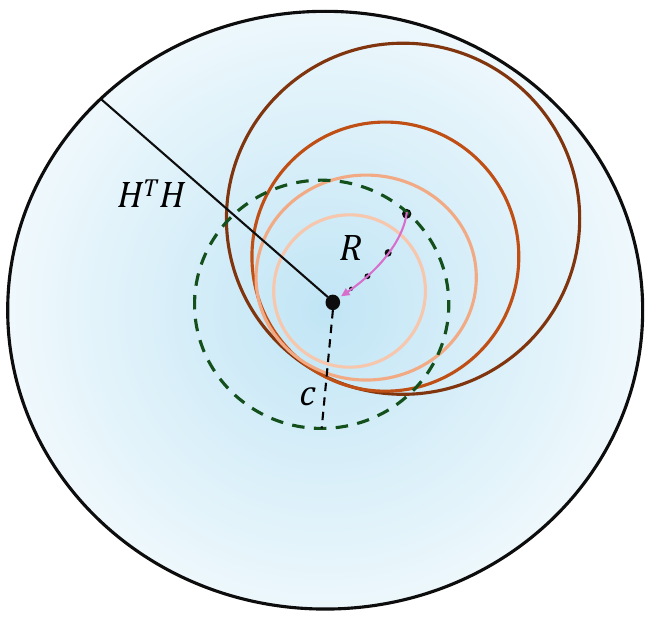}
\caption{Illustration of regularization. The largest circle indicates the range of possible solutions provided by the original neural network with the information matrix $H^TH$ for a given training data set. The center of the largest circle is the theoretical optimal solution. The smaller circles show the regularization solutions under different regularization matrices $R$. The radius of the dotted circle is the bias for a given regularization order $c$. A regularization solution with a proper $R$ provides the trade-off between the generalizability and the estimation bias. Utilizing the proposed HR method, as the regularization matrix moves toward the optimal parameter, the regularization solution approaches the optimal solution while pursuing the smallest condition number of the new matrix $H^TH+R$, in other words, maximizing the generalizability of the neural network.}
\label{relationship_theoretical_optimal_regularization}
\end{figure}
Hence, the regularized solution is an approximation solution based on the approximation inverse mapping. For different regularization matrices and different regularization orders, the relationship between the theoretical optimal solution and the regularized solution is presented in Fig. \ref{relationship_theoretical_optimal_regularization}. The figure illustrates that the regularization solution approaches the optimal solution with a proper regularization matrix while pursuing maximal generalizability. Note that the estimation bias is calculable. Hence, it is possible to correct the bias.

The proposed HR method indicates that the regularization matrix can be used as a contraction factor, where the starting point of the contraction is jointly determined by the given regularization order, the information matrix $H^TH$, and an initial regularization matrix. This property is similar to gradient descent, while HR additionally ensures a stable solution and a calculable bounded error. For a given data set, a lower regularization order corresponds to a greater distance between the regularization solution and the optimal solution, according to equation \eqref{estimation_error_bound}. As the regularization matrix approaches the optimal value, the regularization solution behaves closer to the optimal solution, and the neural network admits better generalizability.

The information matrix $H^TH$ is also called the Hessian matrix, or the inverse of the covariance matrix associated with the data noise in some inverse problems\cite{LowRank2023Baes}. For a given data set, the noise in the data results in various solutions of the trainable weights. A neural network model with desirable generalization performance provides a stable estimation of the trainable weights close to the theoretical optimal solution, even if there is noise in the dataset. The model trained with similar data sets observes a small fluctuation in estimations of the trainable weights, indicating that the learned knowledge is similar and the learning from data is effective.

From Proposition \ref{TR_HR_regularization_parameter}, it follows that for some regularization matrix selections, the maximization of the generalizability of a neural network is realized by denying the effectiveness of the training data. However, in some problems, this property can improve neural network training by effectively acting as a mechanism for forgetting. For instance, if the training data in the early stage of the training is not very reliable, and the reliability of training data increases as the training progresses, which is common in learning-based control problems, then the forgetting of previous data is a beneficial property.

\subsection{Incremental high-order regularization}
This incremental high-order regularization (IHR) algorithm updates the trainable weights with a new batch of training data $\left(H_{ic}, Y_{ic}\right)$ in a smooth manner similar to the incremental ELM. At the time instant $t + 1$, the new information matrix is
\begin{equation}
A_{im,t+1} = H_{t+1}^T H_{t+1} + R = H_{t}^T H_{t}+ R + H_{ic}^{T } H_{ic}.
\end{equation}
Denote the information update factor by
\begin{equation}
K_{t+1}(c) = F_{aim,t+1} F_{aim,t}^{-1},
\end{equation}
where $F_{aim,t}$ and $F_{aim,t+1}$ are the inverse approximation of the mapping $H_{t}^T H_{t}$ and $H_{t+1}^T H_{t+1}$, respectively. Then the HR solution at the time instant $t + 1$ is
\begin{equation}
\begin{aligned}
\hat \beta_{t+1} &= F_{aim,t+1} H_{t+1}^T Y_{t+1} \\
&= F_{aim,t+1} (H_{t}^T Y_{t}+ H_{ic}^T Y_{ic}) \\
&= K_{t+1}(c) \hat \beta_{t}+ F_{aim,t+1}H_{ic}^T Y_{ic},
\end{aligned}
\end{equation}
where $\hat \beta_{t}$ is the HR solution at the time instant $t$.

The IHR algorithm utilizes the information matrix and the trainable weights at the time instant $t$ and the new batch of training data to update the weights. The IHR algorithm is presented in Algorithm
\ref{alg:IHRML}.

\begin{algorithm}[H]
\caption{The proposed IHR neural network}\label{alg:IHRML}
\begin{algorithmic}
\STATE 
\STATE 1 Initialize the network with random weights and bias
\STATE 2 Set $step0 \leftarrow$ $True$
\STATE 3 \textbf{For} episode=1 to maximum episode \textbf{do}
\STATE 4 \hspace{0.1cm} Initialize state $s_t \leftarrow s_0$
\STATE 5 \hspace{0.1cm} \textbf{While} state $s_t$ is non-terminal \textbf{do}
\STATE 6 \hspace{0.4cm} Execute $a_t$, observe $r, s_{t+1}$, and construct $H$
\STATE 7 \hspace{0.4cm} \textbf{If} $step0==$ {\text{True}} \textbf{then}
\STATE 8 \hspace{0.7cm}$\text{Construct } R, c=1$
\STATE 9 \hspace{0.7cm}$ A_{1} = H^T H + R, A_{im,t} = A_{1} $
\STATE 10 \hspace{0.45cm} \textbf{If} $rank(H^TH)<n$ \textbf{then}
\STATE 11 
\hspace{0.7cm}$F(R)=H^TH A_{im,t}^{-1}$
\STATE 12 \hspace{0.45cm} \textbf{Else}
\STATE 13 \hspace{0.7cm}$F(R)=RA_{im,t}^{-1}$
\STATE 14 \hspace{0.45cm} \textbf{End if}
\STATE 15 \hspace{0.55cm}$ \hat \beta{_{t}}=A_{im,t}^{-1} \sum_{i=0}^c F^i(R) H^T Y $
\STATE 16 \hspace{0.5cm} Set $step0 \leftarrow$ $False$
\STATE 17 \hspace{0.25cm} \textbf{Else}
\STATE 18 \hspace{0.7cm}$ F_{aim,t} = A_{im,t}^{-1} \sum_{i=0}^c F^i(R)$
\STATE 19 \hspace{0.7cm}$ A_{im,t+1}^{-1} = (A_{im,t} + H^TH)^{-1}$
\STATE 20 \hspace{0.7cm}$ K_{t+1}(c) = A_{im,t+1}^{-1} \sum_{i=0}^c F^i(R) F_{aim,t}^{-1}$
\STATE 21 \hspace{0.7cm}$ \hat \beta_{t+1} = K_{t+1}(c) \hat \beta_{t} + F_{aim,t+1}H^{T }Y$
\STATE 22 \hspace{0.25cm} \textbf{End if}
\STATE 23 \hspace{0.1cm} \textbf{End while}
\STATE 24 \hspace{-0.1cm} \textbf{End for}
\end{algorithmic}
\end{algorithm}

The regularization solutions are biased. One can correct the bias using the insight of the regularization we just explained. Similar to the derivation of the IHR solution, denote the bias update factor by
\begin{equation}
K_{ar,t+1}(c) = F_{ar,t+1} F_{ar,t}^{-1},
\end{equation}
where $F_{ar,t}$ and $F_{ar,t+1}$ are the residuals of the general inverse approximation of the mapping $H_{t}^T H_{t}$ and $H_{t+1}^T H_{t+1}$, respectively. Then the bias of the HR solution at the time instant $t + 1$ by
\begin{equation}
\begin{aligned}
\Delta \beta_{t+1} 
&= K_{ar,t+1}(c) \Delta \beta_{t}+ F_{ar,t+1}H_{ic}^T Y_{ic},
\end{aligned}
\end{equation}
where $\Delta \beta_{t}$ is the bias of the HR solution at the time instant $t$. Note that since $H_{t+1}^T H_{t+1}$ is ill-conditioned, $F_{ar,t}$ is also ill-conditioned. Correcting the bias of the regularization solution directly with that bias at the beginning of the training leads to an unstable solution. One can correct it with an approximation bias or after some episodes when the new information matrix is well-conditioned. The bias can be ignored if the approximation residual is sufficiently small.

With the same regularization parameter as EQLM, while setting the regularization order to 1 and keeping the rest of the hyperparameters unchanged, the algorithm of the proposed HR to train a neural network for a learning-based control problem is presented in Algorithm \ref{alg:HRML}, which is a simplification of Algorithm
\ref{alg:IHRML}. The main difference between these two algorithms is that during the update processes in Algorithm \ref{alg:HRML}, it fixes $c$ to be 0. The approximate updates in Algorithm \ref{alg:HRML} or any similar approximation updates work, as long as for these approximation update processes, the high-order cases are approximating to the case $c=0$, and the spectral radius condition guarantees the requirement of approximation.

\begin{algorithm}[H]
\caption{The proposed HR neural network}\label{alg:HRML}
\begin{algorithmic}
\STATE 
\STATE 1 Initialize the network with random weights and bias
\STATE 2 Set $step0 \leftarrow$ $True$
\STATE 3 \textbf{For} episode=1 to maximum episode \textbf{do}
\STATE 4 \hspace{0.1cm} Initialize state $s_t \leftarrow s_0$
\STATE 5 \hspace{0.1cm} \textbf{While} state $s_t$ is non-terminal \textbf{do}
\STATE 6 \hspace{0.4cm} Execute $a_t$, observe $r, s_{t+1}$, and construct $H$
\STATE 7 \hspace{0.4cm} \textbf{If} $step0==$ {\text{True}} \textbf{then}
\STATE 8 \hspace{0.7cm}$ A_{1} = H^T H+ \frac{I}{\mu}, \bar \mu =\frac{1}{\mu} $
\STATE 9 \hspace{0.6cm} $R = \bar \mu I, c=1$
\STATE 10 \hspace{0.45cm} \textbf{If} $rank(H^TH)<n$ \textbf{then}
\STATE 11 
\hspace{0.7cm}$F(R)=H^TH\left(H^TH+R\right)^{-1}$
\STATE 12 \hspace{0.45cm} \textbf{Else}
\STATE 13 \hspace{0.7cm}$F(R)=R\left(H^TH+R\right)^{-1}$
\STATE 14 \hspace{0.45cm} \textbf{End if}
\STATE 15 \hspace{0.55cm}$ \hat \beta{_{t}}=\left(H^T H+R\right)^{-1} \sum_{i=0}^c F^i(R) H^T Y $
\STATE 16 \hspace{0.55cm}$ A_{t}^{\dagger } = A_{1}^{\dagger} \sum_{i=0}^c F^i(R)$
\STATE 17 \hspace{0.5cm} Set $step0 \leftarrow$ $False$
\STATE 18 \hspace{0.25cm} \textbf{Else}
\STATE 19 \hspace{0.7cm}$ K_{t} = I- A_t^{\dagger } H^T \left(H A_t^{\dagger } H^T + I_{n\times n} \right)^{\dagger }H$
\STATE 20 \hspace{0.7cm}$ \hat \beta_{t+1} = K_t \hat \beta_{t} + K_t A_{t}^{\dagger }H^{T }Y$
\STATE 21 \hspace{0.7cm}$ A_{t+1}^{\dagger } = K_{t}A_{t}^{\dagger }$
\STATE 22 \hspace{0.25cm} \textbf{End if}
\STATE 23 \hspace{0.1cm} \textbf{End while}
\STATE 24 \hspace{-0.1cm} \textbf{End for}
\end{algorithmic}
\label{alg1}
\end{algorithm}

%===============================================================================

\section{Experiments}
\label{sec:result}

\subsection{Learning-based control problem}
We test the proposed HR with ELM for a classic control problem in reinforcement learning, the cart-pole problem,  using OpenAI Gym \cite{brockman2016openai}. 

\begin{figure}[!ht]
\centering
\includegraphics[width=2.5in]{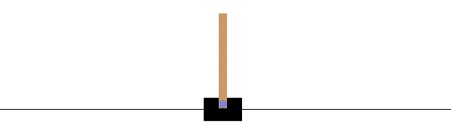}
\caption{A typical control problem in reinforcement learning.}
\label{Cart-Pole}
\end{figure}

\begin{table}[!ht]
\caption{Cart-Pole problem\label{tab:Cart_Pole_problem}}
\centering
\begin{tabular}{c|c|c}
\hline
Problem  & action space  & state space \\
\hline
Cart-Pole & \{left, right\} & \begin{tabular}[c]{@{}c@{}}
\{cart position, cart velocity, \\ pole angle, pole angle velocity\} \end{tabular}
 \\
\hline
\end{tabular}
\end{table}

\begin{table}[!ht]
\caption{Primary setting\label{tab:Primary_setting}}
\centering
\begin{tabular}{c|c}
\hline
Setting  & items  \\
\hline
Termination & \begin{tabular}[c]{@{}c@{}c@{}} pole angle out of $(-12^\circ, 12^\circ)$,\\ cart position out of $(-2.4, 2.4)$,\\ (maximum time step = 200) \end{tabular} \\
\hline
Reward function  & +1  \\
\hline
Starting state   & random value in (-0.05, 0.05)  \\
\hline
Total runs   & 50  \\
\hline
\end{tabular}
\end{table}

\begin{table}[!ht]
\caption{Hyperparameter of the networks\label{tab:Hyperparameter}}
\centering
\begin{tabular}{c|c|c}
\hline
Hyperparameter  & Q-network  & ELM  \\
\hline
Learning rate & 0.0065  & - \\
\hline
Regularization parameter & - & 1.827$e^{-5}$
 \\
\hline
Hidden nodes & 29 & 25 \\
\hline
Initial exploration probability & 0.670 & 0.599 \\
\hline
Episodes to decrease exploration probability & 400 & 360 \\
\hline
Discount factor & 0.99 & 0.93 \\
\hline
Minibatch size & 26 & 2 \\
\hline
Target network update steps & 70 & 48 \\
\hline
Regularization order & - & 0,1 \\
\hline
\end{tabular}
\end{table}
The details of the experiment are given in Figure \ref{Cart-Pole}, Table \ref{tab:Cart_Pole_problem},  Table \ref{tab:Primary_setting} and Table \ref{tab:Hyperparameter}. We adopt most of the settings from the original EQLM \cite{wilson2020novel} in Q-network and EQLM. For the proposed HR extreme learning neural network, we set the regularization order $c=1$, while $c=0$ for EQLM. We train the networks using an Intel i7-12700 (16GB) on a Desktop with an RTX 3060 (12GB) GPU. Each of the networks takes more than one hour to train, including playing the cart-pole game. To show the maximum ability of the neural networks, we lift the time step restriction in some experiments, which is more consistent with many actual problems, and apply the original setting of OpenAI Gym in some experiments for comparison.

\subsection{Results}
We compare the results of the neural network with the gradient descent-based update (Q-network), EQLM, and ELM with the proposed HR method. To show the trend of the curve in each run clearly, a data-smoothing approach is applied to the original reward of each run while plotting the figures. The average reward of 50 runs in each episode is given, and the $95\%$ confidence interval is presented as the shaded area. We first allow long running times. The results are given in Figure \ref{Q-learnig-result}, Figure \ref{EQLM-result}, Figure \ref{HR-ELM-result}, and Table \ref{tab:Q_network_EQLM_HR_Results}. The result of the neural network with gradient descent-based updates (Figure \ref{Q-learnig-result}) shows dramatic fluctuations in every run of the game (the left figure of the results), while the results of EQLM (Figure \ref{EQLM-result}) are more consistent and indicates a more stable result of the average rewards (the right figure of the results). The result of the proposed HR method (Figure \ref{HR-ELM-result}) shows the best performance with the improvement of $31.8 \%$ and $31.3 \%$ in terms of average rewards of the final 50 episodes while comparing the Q-network and EQLM, respectively. The numerical results are summarized in Table \ref{tab:Q_network_EQLM_HR_Results}. The results of all neural networks indicate that the proposed HR method improves the behavior of the neural network in terms of the average reward and the area under the curve (AUC). By analyzing the rewards of every run in the total 50 runs, this improvement is because higher rewards are obtained in some of the runs that attained more than 200 timesteps, which means the corresponding HR network is more stable than EQLM and Q-network.

\begin{figure}[!ht]
\centering
\includegraphics[width=3.55in]{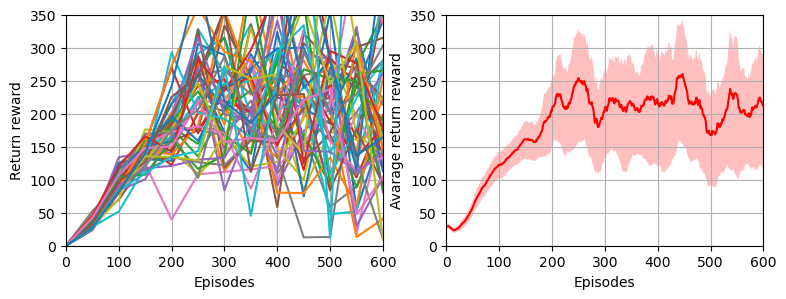}
\caption{Result of the gradient descent-based Q-network in the learning-based control problem. Left: the rewards in each run after the data-smoothness with every 50 episodes. Right: the average rewards of 50 runs in each episode and the $95\%$ confidence interval presenting as the shaded area (same for the rest of the experimental results).}
\label{Q-learnig-result}
\end{figure}

\begin{figure}[!ht]
\centering
\includegraphics[width=3.55in]{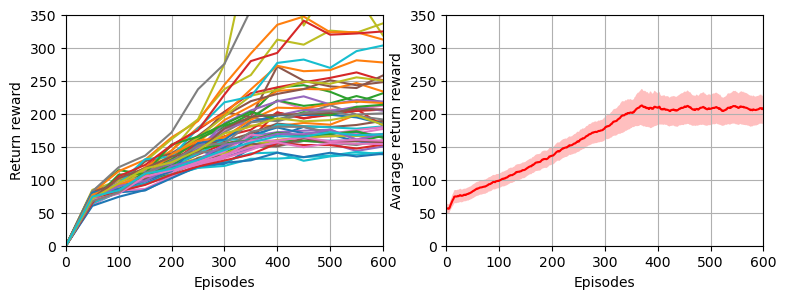}
\caption{Result of the EQLM in the control problem.}
\label{EQLM-result}
\end{figure}

\begin{figure}[!ht]
\centering
\includegraphics[width=3.55in]{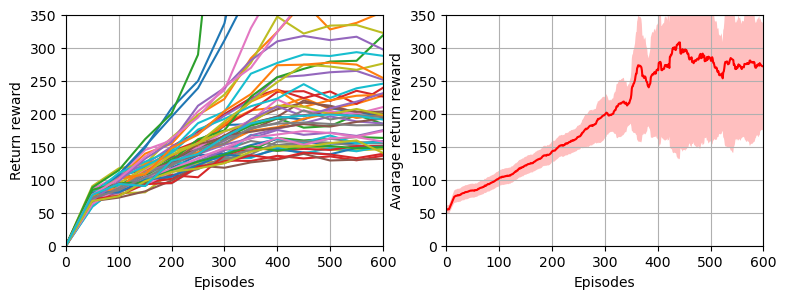}
\caption{Result of the proposed HR method for ELM in the control problem.}
\label{HR-ELM-result}
\end{figure}

\begin{table*}[!ht]
\caption{Results of different methods\label{tab:Q_network_EQLM_HR_Results}}
\centering
\begin{tabular}{cc|c|c|c}
\hline
\multicolumn{2}{c|}{Measure}                                                            & \multicolumn{1}{c|}{Q-network}            & EQLM                   & \multicolumn{1}{c}{Proposed HR}             \\ \hline
\multicolumn{1}{l|}{\multirow{2}{*}{Rewards(50 eps)}} & mean                     & \multicolumn{1}{l|}{207.4 (176.2, 236.3)} & 208.2 (188.51, 224.98) & \multicolumn{1}{l}{\textbf{273.4} (166.1, 340.3)} \\ \cline{2-5} 
\multicolumn{1}{l|}{}                                        & \multicolumn{1}{c|}{std} & 108.7 (86.0, 133.7)                       & 66.2 (42.8, 90.6)      & 323.7 (119.3, 576.4)                      \\ \hline
\multicolumn{1}{c|}{\multirow{2}{*}{AUC($\times 10^3$)}}              & mean                     & 107.8 (103.8, 111.8)                      & 97.2 (90.3, 103.3)     & \textbf{116.3} (86.1, 136.8)                         \\ \cline{2-5} 
\multicolumn{1}{c|}{}                                        & \multicolumn{1}{c|}{std} & 14.8 (11.7, 18.2)                         & 23.5 (16.7, 31.6)      &     95.4 (37.1, 167.6)                                        \\ \hline
\end{tabular}
\end{table*}

\begin{table}[!ht]
\caption{Results of different methods with limit time step\label{tab:Q_network_EQLM_HR_with_limit_Results}}
\centering
\begin{tabular}{cc|c|c}
\hline
\multicolumn{2}{c|}{Measure}                                                            & \multicolumn{1}{c|}{Q-network}                 & \multicolumn{1}{c}{Proposed HR}             \\ \hline
\multicolumn{1}{l|}{\multirow{2}{*}{Rewards(50 eps)}} & mean                     & \multicolumn{1}{l|}{123.8 (112.4, 135.5)}  & \multicolumn{1}{l}{\textbf{175.2} (169.6, 180.9)} \\ \cline{2-4} 
\multicolumn{1}{l|}{}                                        & \multicolumn{1}{c|}{std} & 41.9    (34.7, 50.2)                       & 20.3 (17.4, 24.0)                      \\ \hline
\multicolumn{1}{c|}{\multirow{2}{*}{AUC($\times 10^3$)}}              & mean                     &  71.0 (68.8, 73.5)                     & \textbf{87.5} (84.7, 90.5)                         \\ \cline{2-4} 
\multicolumn{1}{c|}{}                                        & \multicolumn{1}{c|}{std} & 8.3 (6.3, 10.8)                          &     10.4 (8.8, 12.2)                                        \\ \hline
\end{tabular}
\end{table}

As a comparison, we perform the experiments with a limited time step (200) for each episode for the Q-network and the proposed HR method. The results are given in Figure \ref{Q-network-limit-result} and Figure \ref{HR-ELM-limit-result}. For experiments with restricted running time, the proposed HR method shows a more stable result in average rewards and AUC, which is similar to EQLM. The numerical results are provided in Table \ref{tab:Q_network_EQLM_HR_with_limit_Results}, which also shows $41.5\%$ and $23.2\%$ improvements in terms of the average rewards in the final 50 episodes and AUC, respectively. It should be noted that the neural network tends to overfit a dataset at the beginning of the training, where the rewards and timesteps are small, which means that the agent has not explored the environment sufficiently. Training the neural networks with this data can be considered an ill-conditioned problem. The regularization methods enhance the generalizability of the neural networks to improve the performance of the models. 

\begin{figure}[!ht]
\centering
\includegraphics[width=3.55in]{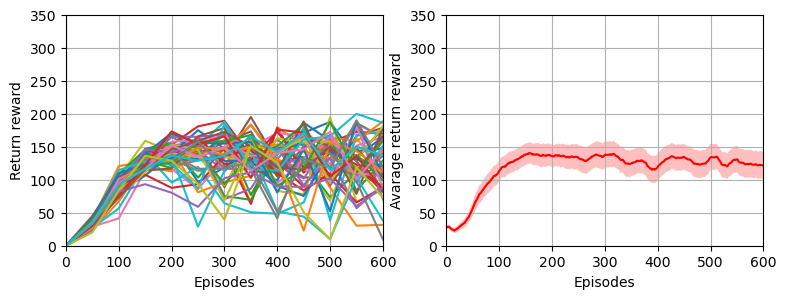}
\caption{Result of the gradient descent-based Q-network under a limited time step, which is set to 200.}
\label{Q-network-limit-result}
\end{figure}

\begin{figure}[!ht]
\centering
\includegraphics[width=3.55in]{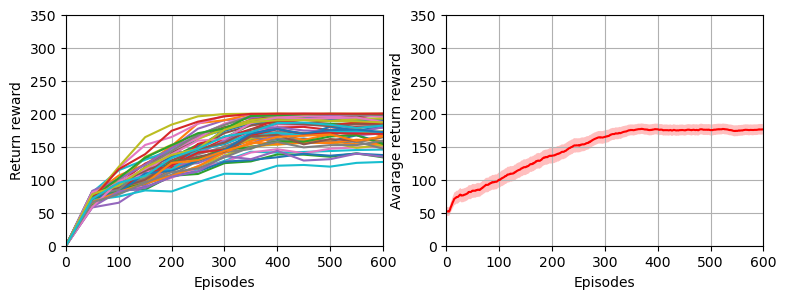}
\caption{Result of the proposed HR method for ELM under a limited time step, which is set to 200.}
\label{HR-ELM-limit-result}
\end{figure}

We perform some experiments to show how the regularization parameter affects the final performance. The experimental results are given in Fig.\ref{Experimental_results_different_mu}. Each point represents an experiment mentioned before. The total runs of one experiment in the learning task are reduced to 20 to speed up. In these experiments, we increase $\bar \mu$ gradually to improve the generalizability by reducing the condition number and the approximation residual.

\begin{figure}[!ht]
\centering
\includegraphics[width=3.35in]{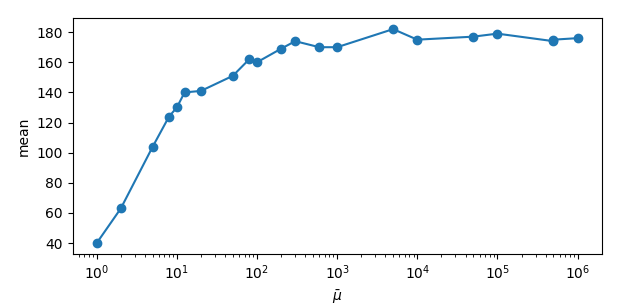}
\caption{Experimental results for different regularization parameter $\bar\mu$ in HR.}
\label{Experimental_results_different_mu}
\end{figure}

We also investigate the improvement of the convergence rate of the training. We find that the exploration probability, a factor that determines how quickly the information about the environment is gathered, is the key to improving the learning rate. The results are presented in Fig. \ref{exploration_probability_convergence_rate}.

\begin{figure*}[!ht]
\centering
\includegraphics[width=7in]{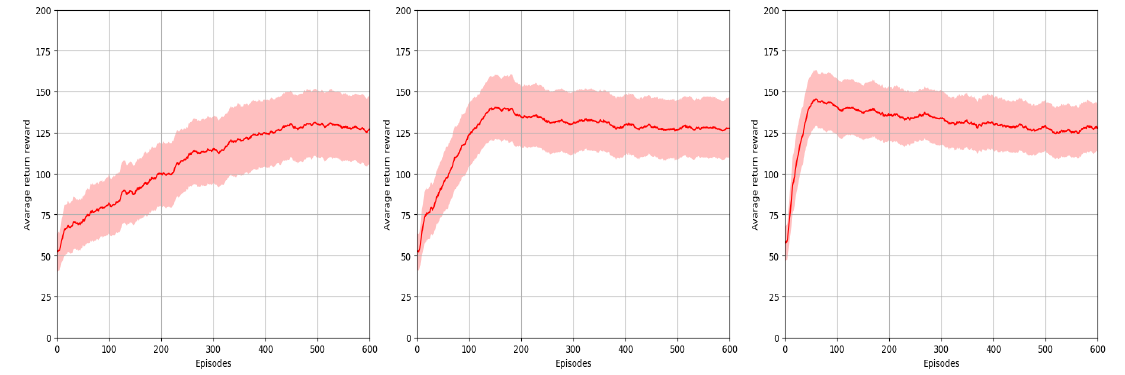}
\caption{The convergence of ELM in the learning-based control problem. From left to right, the number of episodes to decrease the exploration rate is set to 500, 150 and 50 respectively. $\bar \mu$ is set to 10 in these experiments.}
\label{exploration_probability_convergence_rate}
\end{figure*}

\section{Conclusion}
\label{sec:conclusion}
In this work, we propose a novel high-order regularization for training neural networks to improve the behaviors of the extreme learning machine. The proposed high-order regularization is used to calculate the output weights, which link the hidden layer and the output in the ELM network. The proposed HR provides more explainable and reliable results in machine learning, which helps to understand the training process of the neural network, and also benefits performance improvement for solving learning-based control problems. We present a theoretical explanation of regularization in machine learning, which admits the maximization of the generalizability of the neural network.

We explore learning, memory, and forgetting mechanisms in the training process of a neural network. Regularization plays a key role in learning and memory during the first stage of the training process, and then, after most of the information in the training data is learned, regularization plays a role as forgetting. The experimental results verify our conclusions on regularization and demonstrate the superior performance of the proposed HR method, with performance enhancement of the neural network compared with Q-networks and EQLM.

The proposed HR applies to any linear mapping with a tractable approximation inverse mapping. The differential operator used in gradient descent-based neural networks is also a linear mapping, and the learning rate used in the gradient descent update affects the regularization factor in the approximation inverse mapping. The proposed HR provides insight into finding artificial general intelligence. Our future work includes the exploration of more explainable machine learning frameworks based on the proposed methods, adapting the proposed regularization in various existing neural network models, and the application of HR methods to inverse problems such as robot localization and control problems.

\section*{Acknowledgments}
The work was supported in part by the Netherlands Organization for Scientific Research (NWO-Vici-19902), and the China Scholarship Council.

{\appendices
\section{Proof of the regularization error bounds}\label{two_bounds_of_regularization_error}
Starts from the matrix series of $F(R)$, after truncating the terms after $c+1$ items, and note that $\left(I-F(R)\right)^{-1}$ in the the approximation residual $F_{ar}(R)$ has
\begin{equation}
\begin{aligned}
\frac{1}{1-\lambda_{min}(F(R))} \leq \left\|\left(I-F(R)\right)^{-1}\right\| \leq \frac{1}{1-\lambda_{max }(F(R))}.
\end{aligned}
\end{equation}
Then one can obtain the conclusion.

\section{Proof of the convexity}\label{convexity_of_factor_F}
Take $R \in S_+^n$, $H^T H \in S_{++}^n$, and note that $R$ and $H^T H$ commute, then $R$, $H^T H +R$ and $(H^T H +R)^{-1}$ are commutative. Hence, one has $R(H^T H +R)
^{-1} \in S_{+}^n$. 

Take $R_1,R_2 \in S_+^n$, $a\in [0,1]$, and note that $H^T H+R \succeq H^T H$, where $\succeq$ defines the Loewner partial ordering. Then one has
\begin{equation}
\begin{aligned}
&\left(a R_1+(1-a) R_2\right)\left(H^T H+\left(a R_1+(1-a) R_2\right)\right)^{-1}\\
=&a R_1\left(H^T H+\left(a R_1+(1-a) R_2\right)\right)^{-1}+\\
&(1-a) R_2\left(H^T H+\left(a R_1+ 
(1-a) R_2\right)\right)^{-1}\\
 \preceq &a R_1\left(H^T H+a R_1\right)^{-1}+ (1-a) R_2\left(H^T H+(1-a) R_2\right)^{-1}
\end{aligned}
\end{equation}
then the function $R\left(H^T H+R\right)^{-1}$ is a convex function with respected to $R$.

Then, we complete the proof.
}

% \bibliographystyle{IEEEtran}
% \bibliography{references}

\vfill

\end{document}